\documentclass[letterpaper, 10 pt, conference]{ieeeconf}  

\IEEEoverridecommandlockouts                              

\usepackage{amsmath,amssymb,amsfonts}

\usepackage{algorithmic}
\usepackage{graphicx}
\usepackage{textcomp}
\usepackage{xcolor}
\usepackage{dsfont}
\usepackage{accents}
\usepackage{float}
\usepackage[caption=false,font=footnotesize]{subfig}
\usepackage{tikz}
\usetikzlibrary{positioning}
\usepackage{calc}
\usepackage{xfrac}

\definecolor{myred}{rgb}{.85, .15, .29}

\newcommand{\R}{{\mathbb{R}}}

\DeclareMathOperator{\Tr}{Tr}

\newtheorem{theorem}{Theorem}[section]

\newtheorem{lemma}[theorem]{Lemma}

\newtheorem{proposition}[theorem]{Proposition}

\title{\LARGE \bf
Threshold Decision-Making Dynamics Adaptive to Physical Constraints and Changing Environment
}

\author{Giovanna Amorim$^{1}$, María Santos$^{1}$, Shinkyu Park$^{2}$, Alessio Franci$^{3}$, and Naomi Ehrich Leonard$^{1}$
\thanks{This work was supported in part by ONR grant N00014-19-1-2556.}
\thanks{$^{1}$ G. Amorim, M. Santos, N.E. Leonard are with the Dept. of Mechanical and Aerospace Engineering, Princeton Univ., Princeton, NJ 08544 USA, \{{\tt\small giamorim, maria.santos, naomi}\}{\tt\small @princeton.edu}}%
\thanks{$^{2}$ S. Park is with King Abdullah Univ. of Science and Technology
(KAUST), Computer, Electrical and Math. Science and Eng. Div., Thuwal 23955-6900, Saudi Arabia, {\tt\small shinkyu.park@kaust.edu.sa}}%
\thanks{$^{3}$ A. Franci is with the Department of Electrical Engineering and Computer Science, University of Liege, 10 Allée de la Découverte, 4000 Liège, Belgium, and with the WEL Research Institute, Wavre, Belgium, {\tt\small afranci@uliege.be}}%
}

\begin{document}

\maketitle
\thispagestyle{empty}
\pagestyle{empty}
\begin{abstract}
We propose a threshold decision-making framework for controlling the physical dynamics of an agent switching between two spatial tasks. Our framework couples a nonlinear opinion dynamics model that represents the evolution of an agent's preference for a particular task with the physical dynamics of the agent. We prove the bifurcation that governs the behavior of the coupled dynamics. We show by means of the bifurcation behavior how the coupled dynamics are adaptive to the physical constraints of the agent. We also show how the bifurcation can be modulated to allow the agent to  switch tasks based on thresholds adaptive to environmental conditions. We illustrate the benefits of the approach through a multi-robot task allocation application for trash collection. 

\end{abstract}

\section{INTRODUCTION} 
\label{sec:intro}

A simple way to model decision-making of an agent selecting one of multiple options is to use a thresholding mechanism:
The agent makes a decision when a specified variable crosses a static or dynamic threshold. 
This is known as a threshold decision model. While easy to implement on physical systems, the model does not typically account for control of and changes in the agent's physical state.

We propose a threshold decision-making framework for 
agents dynamically choosing 
between two spatial tasks that accounts for the physical state of the agents and does not require a communication network. We use the continuous-time nonlinear opinion dynamics model (NOD)  presented in \cite{bizyaeva2022TAC} for each agent and close the loop 
with the agent's physical dynamics. 
NOD allows agents to form opinions rapidly and reliably in a variety of applications, e.g. \cite{cathcart2023proactive, park2021tuning, hu2023emergent, bizyaeva2022switching}. The dynamics are useful for application in complex environments because of their tunable sensitivity to input \cite{bizyaeva2022TAC,Leonard_ARpaper2024}. Yet it remains an open question how best to systematize the tuning of sensitivity in closed-loop control of physical systems. 

We present new results on  leveraging the tunable sensitivity of NOD for adaptation to physical constraints and changing environment by a group of agents carrying out threshold decision-making for task allocation. Our  approach combines the ease of implementation of  threshold models with the fast and flexible decision-making of NOD.

In our approach each agent is responsive to observed changes in the environment. Since these observed changes may reflect how other agents have affected the environment, agents need not to rely on communication with one another. 
 Benefits of not requiring a communication network include easier scaling of the group and robustness to individual failures.
Our approach is distinguished from the use of NOD in multi-agent task allocation in \cite{bizyaeva2022switching} where physical dynamics are not considered and a communication network is required. 



We let an agent's opinion over a task represent its preference for executing that task. We close the loop between the agent's opinion dynamics and its physical dynamics by letting the agent's opinion inform control of its physical state and the agent's physical state inform evolution of its opinion. The thresholding scheme we use is motivated by a simple scheme used in the literature: Each agent tracks its own evolving efficiency in completing its current task. If its efficiency falls below a threshold, the agent switches to the other task and begins again to track its efficiency on the new task. Approaches using this scheme include \cite{krieger2000threshold, agassounon2002threshold, castello2016threshold, Nidhi2007Threshold}; however, they do not address how an agent's physical state should respond to its decision to switch tasks.  

The task allocation problem we address is from the class of multi-robot tasks with single-task robots~\cite{gerkey2004taxonomy}. Examples of existing approaches to this class of task allocation problems include centralized and decentralized market-based solutions \cite{dias2004market, dias2006market, liu2005market}. These approaches can handle dynamic environments; however, they 
can be hard to scale and rely on direct communication. Other works utilize game-theoretic approaches to produce strategies in which agents communicate to choose utility-maximizing tasks in a decentralized fashion \cite{saad2011game, jang2018game} and a dynamic environment~\cite{ShinkyuICRA}. 


\begin{figure*}
    \centering
    \subfloat[$t=0$\label{subfig:dc_t_0}]
    {\includegraphics[width=.33\textwidth]{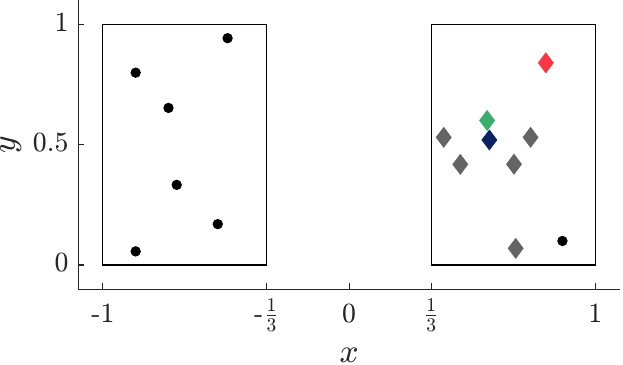}}
    \hfill
    \subfloat[$t=25$\label{subfig:dc_t_50}]
    {\includegraphics[width=.33\textwidth]{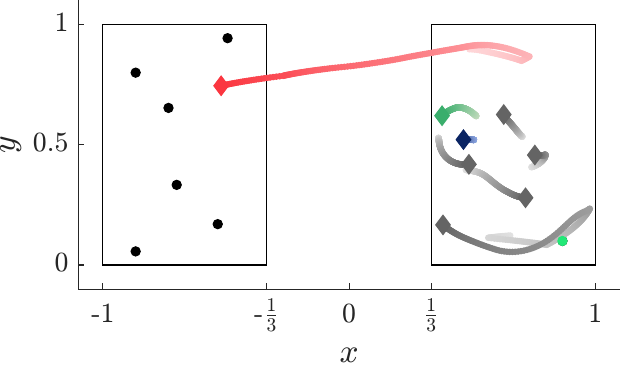}}
    \hfill
    \subfloat[$t=100$\label{subfig:dc_t_150}]
    {\includegraphics[width=.33\textwidth]{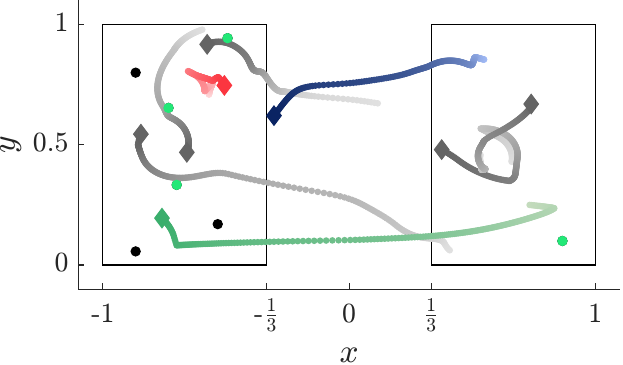}}
    \\
    \subfloat[Trajectory of the red agent.\label{subfig:dc_red}]
    {\includegraphics[width=.33\textwidth]{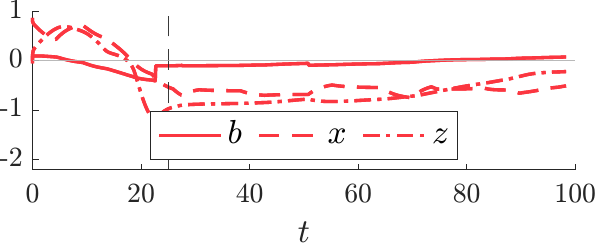}}
    \hfill
    \subfloat[Trajectory of the green agent.\label{subfig:dc_green}]
    {\includegraphics[width=.33\textwidth]{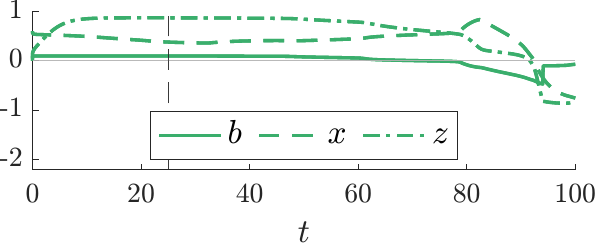}}
    \hfill
    \subfloat[Trajectory of the blue agent.\label{subfig:dc_blue}]
    {\includegraphics[width=.33\textwidth]{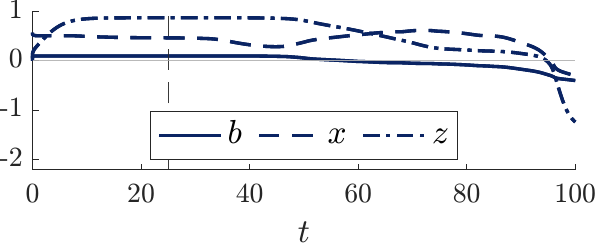}}
    \caption{
    Application of the coupled NOD and physical dynamics model in \eqref{eq:decision_making_model} to multi-robot trash collection across two trash patches. \protect\subref{subfig:dc_t_0}-\protect\subref{subfig:dc_t_150} show the positions of the robots (depicted as diamonds) at $t=0$, $t=25$ and $t=100$s. The patch boundaries are denoted by the rectangles and the uncollected (collected) trash pieces are shown in black (green). The bottom row shows the evolution of the opinion $z_i$, the position $x_i$ and input $b_i$, for the robots highlighted in red, blue, and green in the top row. 
    Parameters: $q_{\min} = 1.5$, $u = 1.3$, $d = 1$, $K_z = 2$, $l = 1$, $\sigma = 0.1$, $k = 10$, $K_y = 0.15$.}
    \label{fig:declustering}
\end{figure*}

Our contributions are as follows. First, we propose a threshold decision-making framework for coupling NOD to the physical dynamics of an agent choosing between two spatially divided tasks. Second, we prove that a bifurcation governs opinion formation 
and that the bifurcation behavior can be used for threshold decision-making. Third, we prove that the coupling allows the model to account for the agent’s physical constraints by implicitly modulating the switching threshold in response to the agent’s physical dynamics. Fourth, we prove how the threshold can be adaptive to environmental changes.
Fifth, we validate the theory with simulations of multi-agent task allocation for trash collection.
 

In Section \ref{sec:model} we introduce the decision-making dynamics. In Section \ref{sec:analysis} we establish the governing bifurcation.
 In Section \ref{sec:application} we illustrate the benefits of our approach to task allocation problems. We provide final remarks 
 in Section \ref{sec:discussion}.

\section{COUPLED NONLINEAR OPINION DYNAMICS AND PHYSICAL DYNAMICS}
\label{sec:model}
We consider an agent choosing between two spatially separated tasks, which requires an agent to travel when switching tasks, as illustrated in Fig.~\ref{fig:declustering}. 
Our framework couples the nonlinear opinion dynamics (NOD) model proposed in \cite{bizyaeva2022TAC} and a model of the agent's physical dynamics that determines the agent's motion in executing its current task. 
The framework accounts for the agent's physical limitations and the presence of obstacles or mechanical failures as well as environmental inputs.
We review NOD in Section~\ref{sub:NOD}.  We define the physical dynamics in Section~\ref{sub:physical} and 
 proposed threshold decision-making framework in Section~\ref{sub:coupledNOD}.

\subsection{Nonlinear Opinion Dynamics}
\label{sub:NOD}
We define the \textit{opinion} $z \in \mathbb{R}$ of an agent as its preference over the two tasks. 
We refer to the region of
task~1 (task~2) as patch~1 (patch~2).  
When $z > 0 \ (< 0)$, the agent favors task 1 (task 2). If $z = 0$, the agent is neutral 
about the tasks. 
The 
commitment of the agent to a task is quantified by 
$|z|$.

We specialize the continuous-time multi-agent multi-option model presented in \cite{bizyaeva2022TAC} to describe how a single agent's opinion over two tasks varies over time: 
\begin{equation} 
\dot{z} = \Tilde{f}(z) := - d\,z +  u S\left( z \right) + b,
\label{eq:uncoupled_NOD}
\end{equation}
where $d > 0$ is a damping coefficient,  $u>0$ is a gain that represents the ``attention'' the agent pays to reinforcing its own opinion, 
and $b\in\R$ is a bias or input in favor of task 1 (task 2) if $b>0$ ($b<0$). We fix $d$ and let $u$ and $b$  be variable.
$S: \R \rightarrow \R$ is an odd saturating function with $S(0) = 0$, $S^\prime(0) = 1$. For simulations we let $S(\cdot) = \tanh (\cdot)$. 

The first term on the right of \eqref{eq:uncoupled_NOD} is a negative feedback that serves to regulate the opinion to the neutral state $z=0$. The second term is a positive feedback that destabilizes $z=0$. The balance of the feedback terms, and thus the stability of $z=0$, is modulated by $u$. The nonlinear saturating function introduces multi-stable solutions so that when $u$ is sufficiently large, $z=0$ is destabilized and $z$ converges to a stable solution where $|z|$ is large.

As explained in \cite{bizyaeva2022TAC,Leonard_ARpaper2024},  two isolated equilibria, corresponding to an opinion in favor of each of the options, arise through a bifurcation from the neutral equilibrium as $u$ increases. When the agent is unbiased ($b=0$) the model (\ref{eq:uncoupled_NOD}) undergoes a supercritical pitchfork bifurcation at $z=0$ and $u = u^* = d$  (as in Fig.~\ref{fig:bif_diagrams_u} top left). For $u < u^*$, the opinion converges to the neutral state $z=0$. When $u > u^*$, $z=0$ becomes unstable, and two branches of locally exponentially stable opinionated equilibria appear. 
If 
$b > 0$ ($b < 0$), the bifurcation unfolds in the direction of the sign of $b$ (as in Fig.~\ref{fig:bif_diagrams_u} top middle (right)). For $b\neq 0$, when $u$ is near $u^*$ 
there is only one stable equilibrium $z^*$, with $\mathrm{sign}(z^*) = \mathrm{sign}(b)$, and the agent's decision-making becomes sensitive to $b$. 

We leverage this feature of NOD for threshold decision-making by treating $b$ as a threshold variable. For a fixed value of $u > u^*$, as $b$ crosses 0 and $|b|$ becomes large enough so that 
there is only one stable equilibrium, the sign of $z^*$ changes 
and the agent switches tasks. The switch results from hysteresis between the two decisions. The size of the bistable region, and thus the threshold on $|b|$, is tuned by $u>u^*$: the larger the $u$ the larger $|b|$ has to be to trigger a switch.

\subsection{Physical Agent Dynamics}
\label{sub:physical}

Suppose without loss of generality that the patches for the two tasks have the same rectangular shape and are equidistant from the $x=0$ line in the two-dimensional Euclidean space. We take $l = 1$ to be the distance from the outermost edge of the patches to the $x=0$ line.
See Fig. \ref{fig:declustering} for an illustration of the patches in a multi-agent trash collection example, where each agent explores one of two patches at a given time as it searches and collects trash.



When an agent selects a task, it moves to the associated patch. It then randomly chooses from a uniform distribution a point $(r_x,r_y)$ in the patch. The waypoint the agent uses to carry out its task is defined as $(\bar x(z), \bar y)$ where $\bar y = r_y$ and $\bar{x}(z) = \rho/l \tanh(kz)$ with $k>0$ and $\rho = |r_x|$.
That is, only the $x$-coordinate of the waypoint depends on the agent's opinion state $z$, as depicted in Fig. \ref{fig:declustering}. 

The physical dynamics evolve as a function of the opinion:
\begin{subequations} \label{eq:uncoupled_phys_xy}
\begin{align} 
\dot{x} &= K_x(\bar{x}(z) - x) = \Tilde{h}(z, x) \label{eq:uncoupled_phys} \\
\dot{y} &= K_y(\bar{y} - y). \label{eq:uncoupled_phys_b}
\end{align}
\end{subequations}
When the agent arrives at the waypoint, it chooses a new waypoint in the same fashion. 
The positive parameter $k$ determines the weight of the opinion on the physical dynamics. The term $\tanh(kz)$ ensures that when the agent's opinion is sufficiently large in magnitude, the opinion dynamics do not impact the waypoint navigation within the patch, i.e., the agent moves very near to the random point $(r_x,r_y)$.  
$K_x, K_y > 0$ 
are proportional gains on the agent's velocity.


\subsection{Coupled Opinion and Physical Dynamics}
\label{sub:coupledNOD}
We propose a decision-making framework that takes the physical dynamics (\ref{eq:uncoupled_phys_xy}) into account in the opinion-forming process by coupling \eqref{eq:uncoupled_NOD} and \eqref{eq:uncoupled_phys_xy}. Since the opinion $z$ affects only \eqref{eq:uncoupled_phys}, we omit \eqref{eq:uncoupled_phys_b} in the coupled model. Let
\begin{subequations} \label{eq:decision_making_model}
\begin{align}
    & \dot{z} = f(z,x) := \Tilde{f}(z) - K_z \eta(z)\left(z - x\right) \label{eq:f}\\
    & \dot{x} = h(z,x) := (1 - \eta(z))\Tilde{h}(z,x)  - \eta(z)\left(x - z\right), \label{eq:h}
\end{align}
\end{subequations}
where $\eta(z) = \exp(-z^2/2\sigma^2)$, $\sigma > 0$, sign$(z(0))=$  sign$(x(0))$, and $K_z>0$ is the coupling weight. The function $\eta(z)$ acts as a smooth switch  between waypoint navigation and task switching.  For $|z|$  large, $\eta (z) \approx 0$ and waypoint navigation is on so that the agent moves to the waypoint. For $|z|$ small, $\eta(z) \approx 1$ and task switching is triggered such that the agent moves towards the origin. 
$\sigma$ determines how small $|z|$ needs to be to trigger a task switch.



To make the agent's decision-making adaptive to the changes in the environment, bias $b$ and attention parameter $u$ can be tuned, as we show in Section~\ref{sec:analysis}. 
The bias $b$, i.e., the threshold variable, is computed in terms of the efficiency of the agent at the current task, as shown in Section \ref{sec:application}.


\section{ANALYSIS OF DECISION-MAKING BEHAVIOR} 
\label{sec:analysis}
\begin{figure} 
    \centering
    \captionsetup[subfigure]{labelformat=empty}
    \subfloat
    {\includegraphics[width=0.32\linewidth]{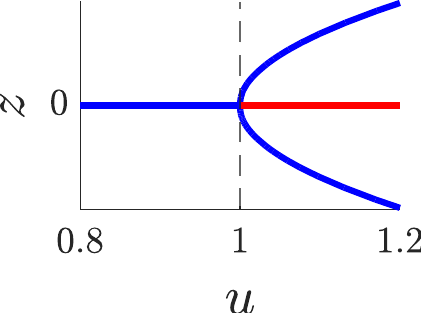}}
    \subfloat
    {\includegraphics[width=0.32\linewidth]{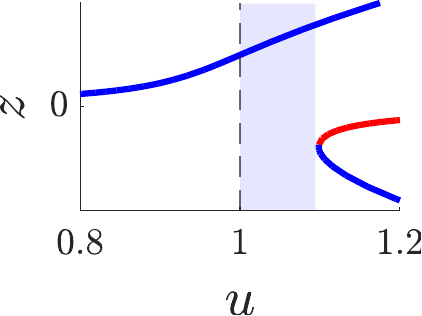}}
    \subfloat
    {\includegraphics[width=0.32\linewidth]{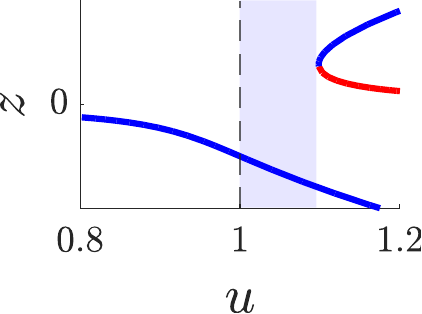}}\\
    \subfloat[\label{subfig:sub_b_0}$b=0$]{\includegraphics[width=0.32\linewidth]{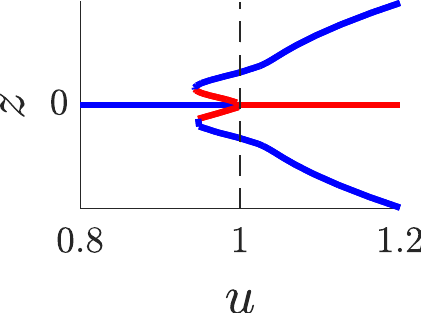}}
    \subfloat[\label{subfig:sup_b_pos}$b>0$]{\includegraphics[width=0.32\linewidth]{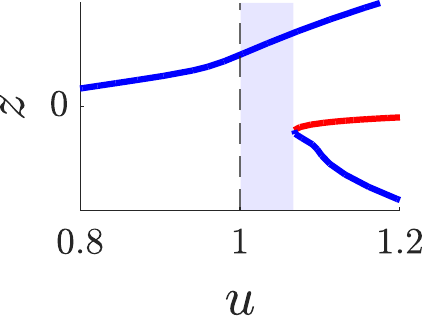}}
    \subfloat[\label{subfig:sup_b_neg}$b<0$]{\includegraphics[width=0.32\linewidth]{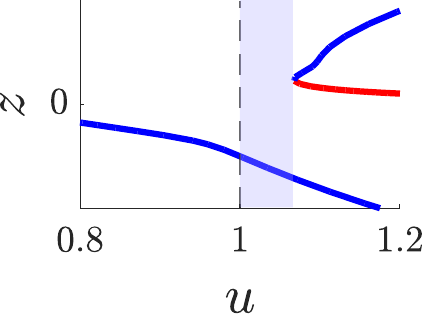}}

    \caption{Bifurcation diagrams for the coupled NOD and physical dynamics of an agent \protect\eqref{eq:decision_making_model}. Each bifurcation diagram plots the equilibrium value of $z$ as a function of a fixed value of $u$. Top row corresponds to parameters for which the system undergoes a symmetric supercritical pitchfork bifurcation and its unfolding. Bottom row corresponds to parameters for which the system undergoes a  symmetric subcritical pitchfork bifurcation with a quintic stabilizing term and its unfolding. Blue (red) lines show stable (unstable) branches of equilibria. The vertical dashed line is $u = u^* $. Regions with only one branch of equilibria for $u > u^*$ are shaded in light blue. Bifurcation diagrams generated with help of MatCont \cite{MatCont}. Parameters: $d = 1$, $b \in \{0,\pm 0.2\}$, $K_x = 3$, $k = 10$, $\sigma = 0.1$, $\rho = 0.5$.}
    \label{fig:bif_diagrams_u}
\end{figure}

In Section~\ref{sub:analysis_NOD} we show how threshold decision-making results through a bifurcation in the coupled dynamics~\eqref{eq:decision_making_model}. In Section~\ref{sec:adaptability_analysis} we show how the threshold decision-making is adaptive to physical constraints and environmental change.

\subsection{Threshold Decision-Making through a Bifurcation}
\label{sub:analysis_NOD}

In a nonlinear dynamical system, a local bifurcation refers to a change in the number and/or stability of solutions as a parameter varies. The parameter and state value at which the change occurs is called a \textit{bifurcation point}.  
The top left bifurcation diagram in Fig. \ref{fig:bif_diagrams_u}, which is a plot of equilibrium values of $z$ as a function of $u$, illustrates a \textit{supercritical} pitchfork bifurcation at bifurcation point $(z,u) = (z^*,u^*) = (0,1)$ in which a (unique) stable equilibrium ($z=0$) becomes unstable and two symmetric equilibria are created. The bottom left diagram of Fig. \ref{fig:bif_diagrams_u} illustrates a \textit{subcritical} pitchfork bifurcation at bifurcation point $(z,u) = (z^*,u^*) = (0,1)$, in which the two symmetric branches that emerge from the stable equilibrium at the bifurcation point are unstable and two further stable branches are created through saddle-node bifurcations. 
This is a subcritical pitchfork with a quintic stabilizing term bifurcation (see \cite{strogatz2000}).

We first show that, with $u$ as the bifurcation parameter, the dynamics \eqref{eq:decision_making_model} undergo either a supercritical pitchfork or a subcritical pitchfork with a quintic stabilizing term bifurcation. 
Observe that the neutral state $\left(z, x \right) = (0, 0)$ is always an equilibrium of \eqref{eq:decision_making_model}.

\begin{lemma} \label{lemma:neutral_eq} (Stability at Neutral Equilibrium) Let $b = 0$.  $J$ is the Jacobian of \eqref{eq:decision_making_model} evaluated at equilibrium $\left(z, x \right) = (0, 0)$. Define $u^* = d$. Then, $\left(z, x \right) = (0, 0)$ is locally exponentially stable for $0 < u < u^*$ and unstable for $u > u^*$. 
\end{lemma}

\begin{proof} 
$J = \big[\begin{smallmatrix} (-d + u - K_z) & K_z\\
        1 & -1
\end{smallmatrix}\big]$
 with eigenvalues $\lambda_{1,2}$.
Then, $\Tr (J) = \lambda_1 + \lambda_2 =  u - d - K_z - 1$ and $\det (J) = \lambda_1\lambda_2 = d - u$. For $ 0 < u < u^* = d$, $\Tr (J) = \lambda_1 + \lambda_2< 0$ and $\det (J) = \lambda_1\lambda_2 > 0 $, therefore $\lambda_1, \lambda_2 < 0$ so the equilibrium $\left(z, x \right) = (0, 0)$ is locally exponentially stable.
For $u > u^* = d$, $\det (J) = \lambda_1\lambda_2 < 0 $, therefore, $\lambda_1 < 0 < \lambda_2$ so the equilibrium $\left(z, x \right) = (0, 0)$  is unstable.
\end{proof}
To show the existence of a symmetric quintic pitchfork bifurcation at $\left(z, x \right) = (0, 0)$, we first reduce the dynamics of \eqref{eq:decision_making_model} to a 1-D scalar differential equation. At equilibrium, $f(z, x) = 0$ holds, so we can solve for $x$ in terms of $z$:
\begin{equation}
   x = z + \mu(z), \ \ \  \ \mu(z) =  \frac{dz - uS(z) + b}{K_z \eta(z)}.
   \label{eq:x_eq}
\end{equation}
Substituting (\ref{eq:x_eq}) with $\boldsymbol{\Gamma} = [d, K_x, K_z, k, \sigma, \rho]$ into $h(z, x) = 0$ gives a scalar equilibrium bifurcation problem: 
\begin{equation}
\begin{split}
       g(z,u,b;
       \boldsymbol{\Gamma})\!  =\!  (1 \! -\! \eta \left( z\right))\Tilde{h}\left(z, z \!+ \mu(z) \right) \! - \! \eta(z) \mu(z)\!  =\!  0.
\end{split}
\label{eq:g}
\end{equation}

\begin{theorem}\label{prop:pitchfork_quint} (Decision-Making Through Bifurcation) Consider the equilibrium bifurcation problem (\ref{eq:g}) with $u$ as the bifurcation parameter and $b = 0$. Suppose that $-2d/K_z - 3K_x(1 - k\rho)/\sigma = 0$. In a neighborhood of ${\left(u, z \right)} = (u^*, 0)$, the bifurcation problem
is strongly equivalent (in the sense of \cite[Definition VI.2.5]{Golubitsky1985}) to the quintic pitchfork bifurcation problem $-z^5 + (u - u^*)z = 0$, where $z$ is the state and $u$ the bifurcation parameter. If $-2d/K_z - 3K_x(1 - k\rho)/\sigma \neq 0$, at the equilibrium ${\left(z, x \right)} = (0, 0)$ with $u^* = d$, the quintic pitchfork bifurcation unfolds into either a supercritical pitchfork bifurcation if  $-2d/K_z - 3K_x(1 - k\rho)/\sigma < 0$ or into a subcritical pitchfork bifurcation with a quintic stabilizing term if $-2d/K_z - 3K_x(1 - k\rho)/\sigma > 0$.
\end{theorem}
\begin{proof} For $b=0$, the bifurcation problem (\ref{eq:g}) is $Z_2$ symmetric with respect to $z \mapsto -z$ because $g(-z, u, 0; \boldsymbol{\Gamma}) = -g(z, u, 0; \boldsymbol{\Gamma})$. Let $g_z$ denote the partial derivative of $g$ with respect to $z$ and similarly for higher-order and mixed derivatives. Following the recognition problem in \cite[Proposition VI.2.14]{Golubitsky1985}, we compute
\begin{align*}
 &g(z^*, u^*, 0; \boldsymbol{\Gamma}) = 0,  \quad \quad
 g_z(z^*, u^*, 0; \boldsymbol{\Gamma}) = 0, \\
 &g_u(z^*, u^*, 0; \boldsymbol{\Gamma}) = 0, \quad\ 
 g_{zu}(z^*, u^*, 0; \boldsymbol{\Gamma}) > 0, \\
 &g_{zzzzz}(z^*, u^*, 0; \boldsymbol{\Gamma})  < 0   \\
 &g_{zzz}(z^*, u^*, 0; \boldsymbol{\Gamma}) 
    = -\tfrac{2d}{K_z} - \tfrac{3K_x}{\sigma} \left( 1 - k \rho \right).
\end{align*} 

Since $g_{zzz}(z^*, u^*, 0; \boldsymbol{\Gamma}) = 0$ for $-2d/K_z - 3K_x(1 - k\rho)/\sigma = 0$, 
by \cite[Proposition VI.2.14]{Golubitsky1985}, the symmetric quintic bifurcation recognition problem is complete. 

If $-2d/K_z - 3K_x(1 - k\rho)/\sigma \neq 0$, then the cubic term $g_{zzz}(z^*, u^*, 0; \boldsymbol{\Gamma}) \neq 0 $. Since the cubic term is non-zero, then by \cite[Theorem VI.3.3]{Golubitsky1985}, the quintic pitchfork unfolds into a supercritical pitchfork if the cubic term is negative or into a subcritical pitchfork bifurcation with a quintic stabilizing term if the cubic term is positive. 
\end{proof}

Theorem \ref{prop:pitchfork_quint} reveals that for $u> u^*$, threshold decision-making can be implemented by varying $b$ like in the uncoupled NOD. This is because it follows from unfolding theory for a pitchfork bifurcation \cite[Ch. I]{Golubitsky1985} that when $b \neq 0$, the symmetric pitchfork unfolds such that only one solution, predicted by the sign of $b$, is stable close to the bifurcation point. The top row of Fig. \ref{fig:bif_diagrams_u} illustrates the bifurcation diagrams for the parameters that satisfy the condition for a supercritical pitchfork and its unfolding. The bottom row of Fig. \ref{fig:bif_diagrams_u} shows the bifurcation diagram for parameters that satisfy the condition for the subcritical pitchfork with quintic stabilizing term. Note that in both the subcritical and supercritical cases, for $u > u^*$, there is only one branch of equilibria in the shaded regions. 

Since we are interested in varying $b$ for threshold decision-making, we now investigate how the system bifurcates when $b$ is the bifurcation parameter. 

\begin{proposition} \label{prop:saddle}(Existence of Saddle Node Bifurcations) Take the equilibrium bifurcation problem (\ref{eq:g}) with $b$ as the bifurcation parameter and $u > u^*$. For $u$ near $u^*$, as $b$ varies, there exist two saddle node bifurcations, one at $(z_1^*, b_1^*)$  with $z_1 > 0 > b_1$ and one at $(z_2^*, b_2^*)$ with $z_2 < 0 < b_2$. 
\end{proposition}
\begin{proof} Obtaining a closed form solution for $(z_1^*, b_1^*)$, $(z_2^*, b_2^*)$ is non-trivial, so we instead examine the normal forms of the supercritical pitchfork ($\dot{z} = -z^3 + (u - d)z + b$), the quintic pitchfork ($\dot{z} = -z^5 + (u - d)z + b$) and the subcritical pitchfork with a quintic stabilizing term ($\dot{z} = -z^5 + z^3 + (u - d)z + b$). By Proposition \ref{prop:pitchfork_quint}, the bifurcation problem near $u = u^* = d$ is strongly equivalent to one of the three pitchfork bifurcation problems.

For the supercritical pitchfork case, let $\hat{g}(z,b,u,d) = -z^3 + (u - d)z + b$. Following the recognition problem in \cite[Theorem IV.2.1]{Golubitsky1985}), we seek $z^*$, $u^*$ such that
\begin{align}
& \hat{g}(z^*,b^*,u,d)  =   -z^3 + (u - d)z + b  = 0 \label{eq:g_hat} \\ 
& \hat{g}_{z}(z^*,b^*,u,d) = -3z^2 + (u - d) = 0 \label{eq:g_hat_z} \\ 
& \hat{g}_{zz}(z^*,b^*,u,d)  = -6z \neq 0 \label{eq:g_hat_zz} \\ 
& \hat{g}_{b}(z^*,b^*,u,d)  \neq 0. \label{eq:g_hat_b}
\end{align}
We first solve for $z^*$ in (\ref{eq:g_hat_z}). Then, we plug $z^*$ into (\ref{eq:g_hat}) to solve for $b^*$. Note that $u > d$. Then, 
\begin{align*}
    &(z_1^*, b_1^*) = \left( \sqrt{\tfrac{u - d}{3}}, -2\left(\tfrac{u -d}{3}\right)^\frac{3}{2}\right) = -(z_2^*, b_2^*)
\end{align*}
Note that conditions (\ref{eq:g_hat_b}) and (\ref{eq:g_hat_zz}) are satisfied since $\hat{g}_{b} = 1 \neq 0$ and $z$. Then, for both saddle points,  $z^*, b^* \neq 0$ and sign$(z^*)$ = -sign$(b^*)$. 
The same process can be repeated for the quintic pitchfork and subcritical pitchfork cases. In both instances, we can verify that there are two saddle points that satisfy $z^*, b^* \neq 0$ and sign$(z^*)$ = -sign$(b^*)$. 
\end{proof}
\begin{figure} 
    \centering
    \subfloat{\includegraphics[width=0.49\linewidth]{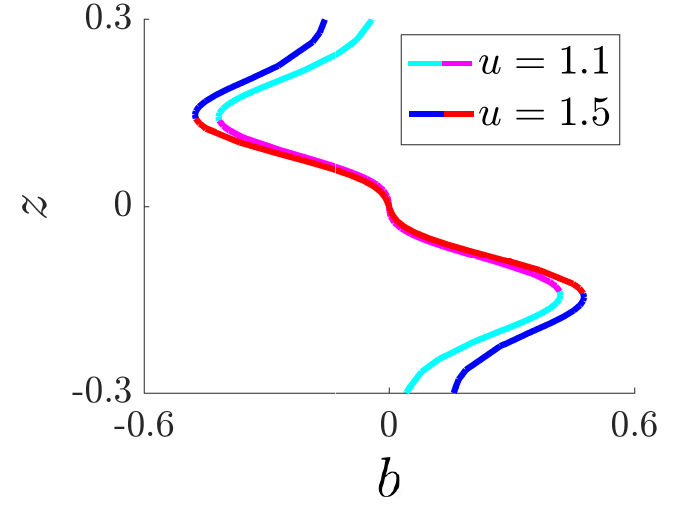}}
    \subfloat{\includegraphics[width=0.49\linewidth]{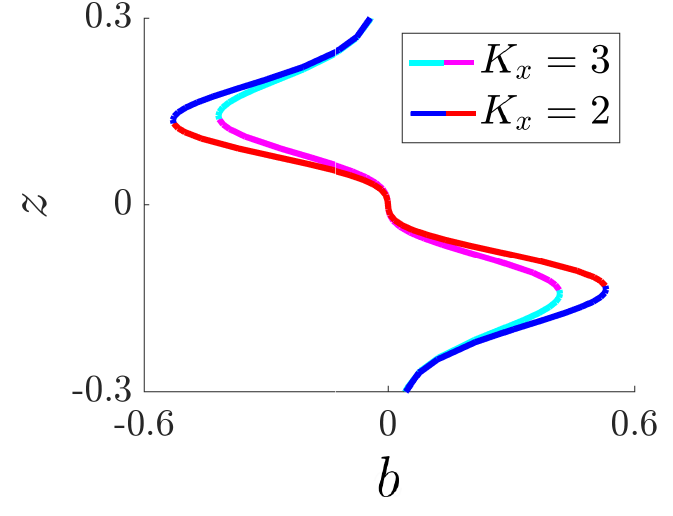}}
    \caption{Bifurcation diagrams for the coupled NOD and physical dynamics of an agent \eqref{eq:decision_making_model} illustrating how $u$ and $K_x$ change the task switching threshold by changing the size of the bistable region. Each bifurcation diagram plots the equilibrium value of $z$ as a function of $b$. Left: bifurcation diagram corresponding to two different values of $u$. Right: bifurcation diagram corresponding to two different values of $K_x$. Blue lines show stable branches of equilibria. Red and magenta lines show unstable branches of equilibria. Bifurcation diagrams generated with MatCont \cite{MatCont}. Parameters: $d = 1$, $k = 10$, $\sigma = 0.1$, $\rho = 0.5$, Left: $K_x = 3$, Right: $u = 1.1$.}
    \label{fig:bif_diagrams_b}
\end{figure}

The saddle nodes serve as the threshold for which the agent switches tasks. As seen in Fig. \ref{fig:bif_diagrams_b}, if the agent starts at the positive branch of stable equilibria, once $b$ gets sufficiently negative and
crosses $b_1^*$, the negative branch becomes the only stable branch. Similarly, if the agent starts at the negative branch, once $b$ gets sufficiently positive and
crosses $b_2^*$, the positive branch becomes the only stable one. Hence, for threshold decision-making, the saddle nodes act like thresholds with $b$ as the threshold parameter. 

\subsection{Threshold Decision-Making Adaptive to Environmental Changes and  Physical Constraints} \label{sec:adaptability_analysis}

We now show how the system adapts to the agent's physical constraints and how tuning a single parameter allows the system to adapt to changing environments. 

\begin{theorem} \label{prop:adap_env} (Adaptability to the Environment Through Saddle Node Bifurcations) Varying the parameter $u$ changes the size of the bistable region by changing the location of the saddle points $(z_1^*, b_1^*)$ and $(z_2^*, b_2^*)$.  Larger $u$ corresponds to a larger bistable region (i.e. larger $|b|$ is required for switching) and smaller $u$ corresponds to a smaller bistable region (i.e. smaller $|b|$ is required for switching). 
\end{theorem}
\begin{proof}  At the saddle node points $(z_1^*, b_1^*)$ and $(z^*_2, b^*_2)$, 
$g(z_1^*, u, b_1^*; \boldsymbol{\Gamma}) = 0$ and  $g(z_2^*, u, b_2^*; \boldsymbol{\Gamma}) = 0$ . Therefore, by the Implicit Function Theorem, the change in location of the saddle points with respect to $u$ can be expressed as
\begin{equation}
    \frac{\partial b^*}{\partial u} = - \left( \frac{\partial g}{\partial b}\right)^{-1} \frac{\partial g}{\partial u} = -S(z). \label{eq:partial_b}
\end{equation}
From Proposition \ref{prop:saddle}, $z^*_1 > 0 > z^*_2$ and $b^*_1 < 0 < b^*_2$, so from (\ref{eq:partial_b}), $\partial b_1^*/{\partial u} < 0$ and $\partial b_2^*/{\partial u} > 0$. Thus, as $u$ increases, the saddle node points move away from the origin.  
\end{proof}

\begin{figure}
    \centering
    \vspace{0.5cm}
    \includegraphics[width=.99\linewidth]{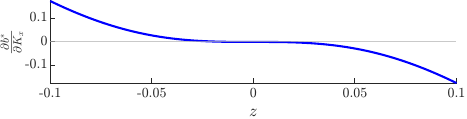} 
    \caption{Evolution of $\partial b^*/\partial K_x$ as a function of $z$. Parameters: $d = 1$, $u = 1.2$, $K_x = 0.05$, $k = 10$, $\sigma = 0.1$, $\rho = 0.8$.}
    \vspace{-0.5cm}
    \label{fig:b_star_partial}
\end{figure}

\begin{proposition} \label{prop:adap_constr}(Adaptability to Physical Constraints Through Saddle Node Bifurcations) Varying the parameter $K_x$ changes the size of the bistable region by changing the location of the saddle points $(z_1^*, b_1^*)$ and $(z_2^*, b_2^*)$.  Larger $K_z$ corresponds to a smaller bistable region (i.e. smaller $|b|$ is needed for switching) and smaller $K_x$ corresponds to a larger bistable region (i.e. larger $|b|$ is needed for switching). 
\end{proposition}

\begin{proof} The proof follows the same method as the one presented in Theorem \ref{prop:adap_env}, but with respect to $K_x$ instead of $u$. Therefore, the change in location of the saddle points with respect to $K_x$ can be expressed as $\frac{\partial b^*}{\partial K_x} = - \left( \frac{\partial g}{\partial b}\right)^{-1} \frac{\partial g}{\partial K_x}$.
Determining the sign of $\frac{\partial b^*}{\partial K_x}$ is non-trivial, so we examine the evolution of $\frac{\partial b^*}{\partial K_x}$ as a function of $z$ for small values of $b^*$. In Fig. \ref{fig:b_star_partial}, we have $\frac{\partial b_1^*}{\partial K_x} < 0$ and $\frac{\partial b_2^*}{\partial K_x} > 0$ for $z^*_1 > 0 > z^*_2$ and $b^*_1 < 0 < b^*_2$. Thus, as $K_x$ increases, the saddle node points moves towards the origin. 
\end{proof}

Theorem~\ref{prop:adap_env} and Proposition  \ref{prop:adap_constr} imply that the decision-making threshold, given by the location of the saddle nodes, can be modulated by $K_x$ and $u$. Fig. \ref{fig:bif_diagrams_b} illustrates how changing both of these parameters have the same global effect: changing the size of the bistable region  effectively changes the decision-making threshold. On the left, we see that increasing $u$ increases the switching threshold. On the right we see that increasing $K_x$, decreases the switching threshold. 

\section{APPLICATIONS TO TASK ALLOCATION}
\label{sec:application}

\begin{figure}
    \centering
    \vspace{0.5cm}
    \subfloat
    {\includegraphics[width=0.9\linewidth]{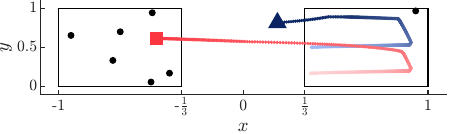}}\\
    \subfloat
    {\includegraphics[width=0.45\linewidth]{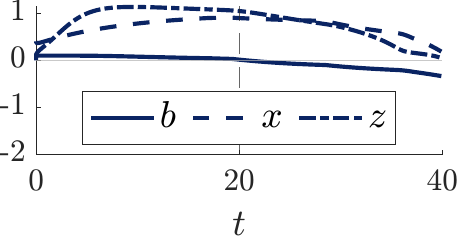}}
    ~~ \subfloat
    {\includegraphics[width=0.45\linewidth]{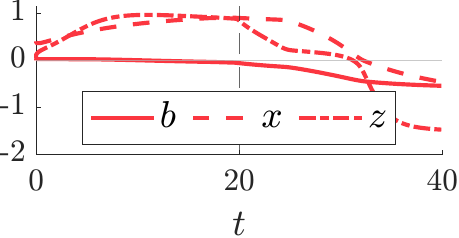}}
    \caption{Two robots (red square and a blue triangle) collecting trash (black circles) across two trash patches (black rectangles) adapt to environmental changes. At $t = 20$, the red robot receives a signal from the environment to decrease $u$ thus increasing its switching threshold. Top: spatial trajectory of the two robots. Bottom: time trajectory of the opinion $z$, the position $x$ and input $b$ of each robot. Parameters: $q_{\min} = 1.5$, $d = 1$, $K_z = 2$, $l = 1$, $\sigma = 0.1$, $k = 10$, $K_y = 0.15$, $K_x = 0.11$, blue robot: $u = 1.3$, red robot before $t = 20$: $u = 1.3$ and after $t = 20$: $u = 1.05$.
    }
    \vspace{-0.5cm}
    \label{fig:adaptabilityEnvt}
\end{figure}

Through simulations, we show 
how the proposed threshold decision-making and physical dynamics models can be used in multi-robot task allocation 
applications to allow robots to individually adapt to changing environments, how it can accommodate the robots' physical constraints, and how it can be used to reduce congestion in regions with clusters of robots (i.e., a large number of robots close together). 

We study a multi-robot trash collection problem in which  trash-collecting robots with low sensory capabilities use the decision-making model in \eqref{eq:decision_making_model} to select which patch to search and collect trash from based on perceived efficiency of the currently selected patch. In this dynamic task allocation problem, patch resources change as a result of robots picking up the trash and because trash can be added to the patches. We consider the efficiency $q\in\R_+$ of a robot to be given by the ratio of 
perceived resource abundance to energy spent: 
\begin{equation*}
    q = \frac{\mbox{\# of trash pieces collected in the patch} + q_0}{\mbox{distance travelled within patch} + \epsilon},
\end{equation*}
where the constant $q_0>0$ ensures that $q > 0$ and $\epsilon >0$ ensures that $q$ is always defined. In practice, $q_0 >> \epsilon$.

We can calculate the appropriate input $b$ to the decision making in model \eqref{eq:decision_making_model} as a function of the efficiency:
\begin{equation*} \label{eq:b_qual}
    b(q) = s(\tanh(q) - \tanh(q_{\min})),
\end{equation*}
where the value of $q$ is reset and $s$ changes sign whenever a robot enters a patch after switching. 
When the robot enters patch 1 (patch 2), $s = 1$ ($s = -1$). In that way, when $q$ falls below $q_{\min}$, $b$ makes the robot start favoring the other patch. 


\paragraph{Adaptability to Environmental Changes} \label{sec:adaptability}
Our approach provides flexibility by allowing robots to adapt their decision-making in response to environment changes, such as changes to global resource availability. These changes can be captured by varying the value of $u$ based on robot measurements or a model of how the resources are changing. 
See Fig. \ref{fig:adaptabilityEnvt} where trash is added to patch 2, and the red 
robot receives a signal to increase $u$, effectively increasing its switching threshold. Then, although red and blue robots have the same $b$, the red
switches patches before the blue 
and heads to the patch which now has more abundant resources. 


\paragraph{Emergent Explore-Exploit Behavior with Heterogeneous Robots} \label{sec:speed}
Seemingly homogeneous robots may still experience some level of heterogeneity since many factors can affect their performance, e.g. battery charge levels, CPU temperatures, motor strain, or payload weight. Additionally, for some applications, a heterogeneous group of robots can be more advantageous. A smaller and faster robot could be better at exploring the environment while a larger and slower one could be better at carrying heavier payloads. 

The coupling term in our decision-making model makes 
it adaptive to the physical constraints of the robot. 
For a group of heterogeneous robots, this feature leads to emergent explore-exploit behavior. 
In Fig. \ref{fig:slowVFast}, given the same $b$, the faster red 
robot switches before the blue 
robot. Therefore, the faster robot presents a more exploratory behavior while the slower one presents a more exploitative behavior. 

\begin{figure}
    \centering
    \vspace{0.5cm}
    \subfloat
    {\includegraphics[width=0.9\linewidth]{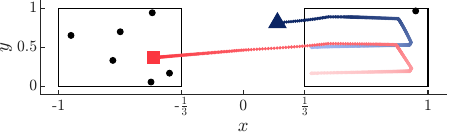}}\\
    \subfloat
    {\includegraphics[width=0.45\linewidth]{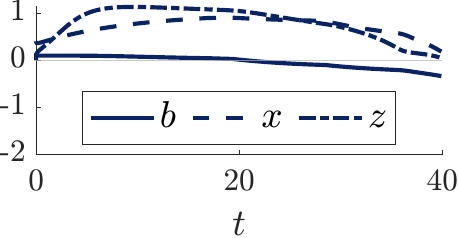}}
    ~~ \subfloat
    {\includegraphics[width=0.45\linewidth]{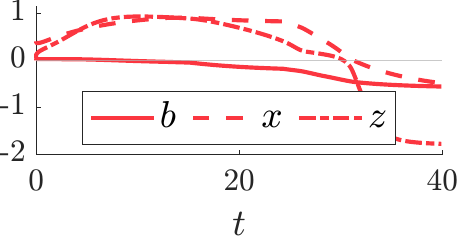}}
    \caption{Two robots with different physical constraints, depicted as a red square and a blue triangle, collecting trash (black circles) across two trash patches (black rectangles). The red robot is faster than the blue robot. Top: spatial trajectory of the two robots. Bottom: time trajectory of the opinion $z$, the position $x$ and input $b$ of each robot. Parameters: $q_{\min} = 1.5$, $u = 1.3$, $d = 1$, $K_z = 2$, $l = 1$, $\sigma = 0.1$, $k = 10$, $K_y = 0.15$, red robot: $K_x = 0.15$. blue robot: $K_x = 0.11$. }
    \label{fig:slowVFast}
\end{figure}

\paragraph{Declustering Behavior} \label{sec:decluster}
The simulation in
Fig. \ref{fig:declustering} illustrates how our model relieves traffic congestion. Some of the robots are initially clustered in the center of patch. The control barrier function based collision avoidance algorithm \cite{wang2017CBF} used slows them down by effectively changing their value of $K_x$. The robots that are not near the cluster are free to move faster, 
so they switch more quickly leaving the patch less crowded. Eventually, the cluster disappears. The blue robot, which was slowed down the most due to it being in the center of the cluster is the last robot to switch patches. 

\section{DISCUSSION AND FINAL REMARKS}
\label{sec:discussion}
We have presented a threshold decision-making framework that allows for flexible task allocation for spatial task applications through the coupling of opinion dynamics to the physical dynamics of the robot. In future work, we aim to generalize the results in \ref{sec:analysis}
and \ref{sec:adaptability_analysis} to non-symmetric patches,  extend the analysis to the multi-option dynamics of \cite{bizyaeva2022TAC} and explore implementation of the algorithm on physical robots.


\bibliographystyle{./bibliography/IEEEtran}
\bibliography{./bibliography/references_short}

\end{document}